\documentclass[journal,twoside,web]{ieeecolor}
\usepackage{lcsys}
\usepackage{cite}

\usepackage{amsmath,amssymb,amsfonts}
\usepackage{amsthm}

\usepackage{url}

\newtheorem{lemma}{Lemma}
\newtheorem{theorem}{Theorem}
\newtheorem{remark}{Remark}
\newtheorem{problem}{Problem}

\usepackage{algorithmic}
\usepackage{graphicx}
\usepackage{textcomp}
\let\labelindent\relax
\usepackage{enumitem}
\usepackage{mathtools}
\def\BibTeX{{\rm B\kern-.05em{\sc i\kern-.025em b}\kern-.08em
    T\kern-.1667em\lower.7ex\hbox{E}\kern-.125emX}}
\begin{document}
\title{Trajectory Tracking with Reachability-Guided Quadratic Programming and Freeze-Resume}

\author{Hossein Gholampour, \IEEEmembership{Student Member, IEEE}, 
        Logan E. Beaver, \IEEEmembership{Member, IEEE}
\thanks{H. Gholampour and L. E. Beaver are with the Department of Mechanical and Aerospace Engineering,
Old Dominion University, Norfolk, VA, USA (e-mail: {mghol004,lbeaver}@odu.edu).}}
\maketitle

\begin{abstract}
Many robotic systems must follow planned paths yet pause safely and resume when people or objects intervene. We present an output–space method for systems whose tracked output can be feedback–linearized to a double integrator (e.g., manipulators). The approach has two parts. \emph{Offline}, we perform a pre-run reachability check to verify that the motion plan respects speed and acceleration magnitude limits. \emph{Online}, we apply a quadratic program to track the motion plan under the same limits. We use a one-step reachability test to bound the maximum disturbance the system is capable of rejecting.
When the state coincides with the reference path we recover perfect tracking in the deterministic case, and we correct errors using KKT-inspired weight. 
We demonstrate that safety stops and unplanned deviations are handled efficiently, and the system returns to the motion plan without replanning.
We demonstrate our system's improved performance over pure pursuit in simulation.

\end{abstract}

\begin{IEEEkeywords}
collaborative robots, constrained motion planning, motion control, optimization and optimal control, physical human--robot interaction, robot safety
\end{IEEEkeywords}

\section{Introduction}
\label{sec:introduction}

\IEEEPARstart{M}{any} modern motion systems must follow planned paths while remaining safe around people and objects in their workspace. In collaborative robotics, this often means intentional or unintentional stoppages—force thresholds, emergency stops, or incidental bumps—that displace the system from its path. After such events, the controller must resume from off–path without significant replanning while respecting state and control constraints in real time.

%Background
Compliance at the interaction layer—impedance or admittance control—remains the standard way to regulate contact forces in human-robot interaction \cite{hogan1985impedance}, and it is widely used in manipulation and grasping \cite{mohammadi2024development} with growing emphasis on safety in cobots \cite{chemweno2020orienting}. For path execution, geometric trackers such as pure pursuit remove explicit time dependence by steering toward a look-ahead point; path-feasibility and time-scaling approaches test whether a planned path can be traversed under speed/acceleration bounds, with reachability-based time-optimal path parameterization (TOPP) has made these checks efficient and robust \cite{pham2018new}. Model predictive control (MPC) offers constraint handling but can be heavy for high-rate, long-horizon use or when frequent stops force repeated warm starts and replans.

%Problem statement - gaps
Two gaps motivate this work. First, pure pursuit typically enforces limits by a posteriori saturation, which creates significant error near high curvature and after pauses \cite{coulter1992,paden2016survey}. Second, reachability-based TOPP assumes uninterrupted motion and does not address mid-run halts \cite{pham2018new}. Real-time trajectories generators produce jerk-limited profiles  \cite{Berscheid2021JerklimitedRT} but usually in joint space with a prescribed time law. Constraint-aware MPC (including learned disturbances) can help \cite{kermanshah2024control}, yet it may require significant computational effort when the goal is to rejoin a path safely and quickly.
Related adaptive/robust strategies for multi-agent systems handle severe uncertainty and guarantee bounded errors \cite{norouzi2025novel, ren2008distributed}, yet they do not target single-agent path rejoining with per-sample feasibility.

We investigate two research questions in this work. \emph{Research Question 1}: Can a one-step reachability test, paired with an adaptive position/velocity weighting, reduce tracking error and limit violations relative to pure pursuit under stoppages and off-path restarts? \emph{Research Question 2}: Can such a scheme run in real time and resume the same path without replanning or overshooting?
To address these questions, we have developed an online tracking controller that can be applied to any system feedback linearizeable to a double integrator.
This includes manipulators in workspace coordinates, quadrotors, and mobile ground robots. %to develop a The significance is generality: the method is formulated in output space and applies to any plant whose tracked output can be feedback-linearized to a double integrator (manipulators, Cartesian stages, and related motion systems), not just a single robot.

We propose a two-part solution. \emph{Offline}, we check the planned path for passability under speed/acceleration \emph{magnitudes} to verify trajectory feasibility before execution, including a robust approach to disturbance estimation. \emph{Online}, we propose a QP-based tracking controller that enforces state and control bounds while tracking the reference trajectory using a look-ahead strategy.
We develop a one-step reachability margin that compares the acceleration required to land on that target with the allowable bound, which also acts as a buffer for bounded disturbances. 
If the system state coincides with the reference trajectory, we prove that we recover perfect tracking in the deterministic case.
Under stochastic disturbances, %target is reachable, the QP acts as a pure position corrector; if not, 
we include a velocity term with an adaptive weight (from a simple KKT-inspired update) to re-enter the reachable set smoothly.
We demonstrate our QP's ability to track trajectories in simulation by including an emergency braking event, which drives the system off of the state trajectory for some time.
%Safety stops are handled uniformly: brake to rest, then resume the same path without replanning.

The remainder of the paper is organized as follows. Sec. \ref{sec:problem} states the model and freeze–resume problem, Sec. \ref{sec:method} derives the reachability-aware QP, Sec. \ref{sec:setup} describes the simulation setup, and Sec. \ref{sec:results} reports results; Sec.\ref{sec:conclusion} concludes.

\section{Problem Formulation} \label{sec:problem}
\subsubsection*{Dynamics and Constraints}
We start at the output level with a relative–degree–two model with small bounded perturbations.
As a running example, consider an $n$~DOF manipulator tracking a task–space (Cartesian) path at its end–effector; the same formulation applies to Cartesian stages and similar systems.

Let $\mathbf p(t),\mathbf v(t)\in\mathbb R^3$ denote output position and velocity and let $\mathbf u(t)\in\mathbb R^3$ be the commanded Cartesian acceleration. We use a relative–degree–two model with small bounded disturbances:

\begin{equation}
  \dot{\mathbf{p}}(t) = \mathbf{v}(t) + \mathbf{n}_p(t), 
  \qquad 
  \dot{\mathbf{v}}(t) = \mathbf{u}(t) + \mathbf{n}_v(t),
  \label{eq:dynamics}
\end{equation}

$\mathbf n_p(t)$ and $\mathbf n_v(t)$ are stochastic process/actuation noises. For robustness, we assume they are almost-surely bounded at each $t$
$\|\mathbf n_p(t)\|\le \varepsilon_p$ and $\|\mathbf n_v(t)\|\le \varepsilon_v$. 
These bounds will be used later only to construct a one–step margin for feasibility decisions; they are not treated as hard constraints.
\begin{equation}
  \|\mathbf{n}_p\| \le \varepsilon_p,\qquad 
  \|\mathbf{n}_v\| \le \varepsilon_v.
  \label{eq:noise-bounds}
\end{equation}

With sampling period $t_s$ and $t_k=k\,t_s$, the forward Euler update is
\begin{equation}
\begin{aligned}
    \mathbf{v}_{k+1} &= \mathbf{v}_k + \big(\mathbf{u}_k + \mathbf{n}_{v,k}\big)\, t_s,\\
    \mathbf{p}_{k+1} &= \mathbf{p}_k + \big(\mathbf{v}_k + \mathbf{n}_{p,k}\big)\, t_s  + \tfrac{1}{2}\big(\mathbf{u}_k + \mathbf{n}_{v,k}\big) t_s^2 .
\end{aligned}
\label{eq:discrete}
\end{equation}

Hardware and safety impose magnitude limits on speed and acceleration,
\begin{equation}
    \|\mathbf{v}(t)\| \le v_{\max}, 
    \qquad
    \|\mathbf{u}(t)\| \le a_{\max}, 
    \label{eq:constraints}
\end{equation}
which we will enforce \emph{inside} the online QP rather than by a posteriori saturation. In discrete time this amounts to choosing $\mathbf u_k$ from the intersection of the acceleration ball $\|\mathbf u_k\|\!\le\!a_{\max}$ and the one-step speed ball $\|\mathbf v_k+t_s\mathbf u_k\|\!\le\!v_{\max}$.

\vspace{0em}
\subsubsection*{Reference Path.}
The desired motion is a smooth, twice–differentiable parametric curve 
(e.g., Bézier or B–spline). The path may come from an offline planner and is fixed during execution.

\begin{equation}
    \mathbf{p}_{\mathrm{ref}} : [0, 1] \to \mathbb{R}^3
\end{equation} \label{ref-path}
which maps a normalized path parameter $s\!\in[0,1]$ to position in $\mathbb{R}^3$.

The tracking problem is to choose $\mathbf{u}(t)$ so that $(\mathbf{p}(t),\mathbf{v}(t))$ follows $\mathbf{p}_{\mathrm{ref}}(s)$ with small position and velocity error, while satisfying \eqref{eq:dynamics}–\eqref{eq:constraints} at all times.

\vspace{0em}
\subsubsection*{Freeze–Resume Scenario.}
We follow the same admittance-style structure as our prior work \cite{gholampour2025mass}. When a safety/interaction event occurs (e.g., human contact, collision avoidance, or emergency-stop), the system may \emph{freeze} by driving
\begin{equation}
    \mathbf{v} \to \mathbf{0}, \quad \mathbf{u} \to \mathbf{0}
\end{equation}

while maintaining compliance at the contact point, until a resume signal is given. On resumption, the controller must rejoin the original path with:
\begin{itemize}
    \item minimal overshoot in position and velocity,
    \item strict adherence to \eqref{eq:constraints},
    \item no replanning of the path.
\end{itemize}
This is challenging because an abrupt halt creates a mismatch to $\mathbf{p}_{\mathrm{ref}}(s)$ and $\mathbf{v}_{\mathrm{ref}}$ that simple trackers may not handle well.

\vspace{0em}
\subsubsection*{Assumptions.}
\begin{enumerate}
    \item \textbf{Output model:}
    In the operating region, the output behaves as in \eqref{eq:dynamics}, i.e., a double-integrator with small bounded perturbations $\mathbf{n}_p(t)$ and $\mathbf{n}_v(t)$.

    \item \textbf{Real-time feasibility:} 
    Computation/communication delays are negligible relative to $t_s$ (the control applies within the sample).

    \item \textbf{Path smoothness:} The reference path is twice–continuously differentiable, $ \mathbf{p}_{\mathrm{ref}} \in \mathcal C^2([0,1];\mathbb R^3)$ (e.g., a B-spline or Bézier constructed to be at least $\mathcal C^2$), so curvature and along-path acceleration are well defined.
    
\end{enumerate}

The above assumptions capture what we need for the method to run in real time and stay safe at the output level. Real-time feasibility keeps the QP inside the sampling budget, and path smoothness ensures curvature, velocity, and acceleration along $\mathbf{p}_{\mathrm{ref}}(s)$ are well defined for our checks.

\section{Solution} \label{sec:method}
Our controller has two parts: (i) an \emph{offline reachability} check that verifies the planned path admits a time–scaling under speed/acceleration limits and flags if the path is not reachable; and (ii) an \emph{online tracking controller} that, at each sample, solves a small QP in output space to compute the acceleration command (Cartesian for manipulators) while enforcing the same bounds. A one–step reachability margin decides whether to act as a position corrector or to blend in a velocity term with an adaptive weight. Freeze events are handled by braking to rest and then resuming the same controller without replanning. 
We present the online controller first, then derive our offline feasibility check using our derived reachability margin.

\vspace{0em}
\subsection*{Online tracking controller}
Given the current position $\mathbf{p}_k$, we first find the closest point on the path
\begin{equation}
    s_c = \arg\min_s \|\mathbf{p}_{\mathrm{ref}}(s) - \mathbf{p}_k\|,
    \label{eq:closest-point}
\end{equation}
then choose a look-ahead shift from the local reference speed
\begin{equation}
    v_{\mathrm{LA}} = \|\mathbf v_{\mathrm{ref}}(s_c)\|, 
    \qquad 
    s_{\mathrm{LA}} = v_{\mathrm{LA}}\, t_s.
    \label{eq:lookahead}
\end{equation}
The look-ahead position and velocity are
\begin{equation}
    \mathbf{p}_{\mathrm{LA}} = \mathbf{p}_{\mathrm{ref}}(s_c + s_{\mathrm{LA}}), 
    \quad
    \mathbf{v}_{\mathrm{LA}} = \mathbf{v}_{\mathrm{ref}}(s_c + s_{\mathrm{LA}}).
    \label{eq:LA-p&v}
\end{equation}

Using forward Euler, the one-step prediction errors are
\begin{align}
\mathbf e_p &= \mathbf{p}_{\mathrm{LA}} - \Big(\mathbf{p}_k + t_s \mathbf{v}_k + \tfrac{1}{2} t_s^2 \mathbf{u}_k\Big), \\
\mathbf e_v &= \mathbf{v}_{\mathrm{LA}} - \Big(\mathbf{v}_k + t_s \mathbf{u}_k\Big).
\label{eq:ep-ev}
\end{align}

\begin{problem} %[Online QP at sample $k$]
\label{prob:qp}
Given the current state $(\mathbf p_k,\mathbf v_k)$, the look-ahead pair 
$(\mathbf p_{\mathrm{LA}},\mathbf v_{\mathrm{LA}})$ from \eqref{eq:LA-p&v}, bounds 
$a_{\max},v_{\max}$ in \eqref{eq:constraints}, and weight $C$, compute the
acceleration vector $\mathbf u_k$ by
\begin{equation}
    \begin{aligned}
    \min_{\mathbf u_k}\quad & J=\|\mathbf e_p\|^2 + C_k\,\|\mathbf e_v\|^2 \\
    \text{s.t.}\quad & \|\mathbf u_k\|\le a_{\max},\qquad \|\mathbf v_k + t_s\mathbf u_k\|\le v_{\max}.
    \end{aligned}
    \label{eq:cost}
\end{equation}
\end{problem}

We enforce the next-step output-space limits through the constraints in \eqref{eq:cost}. Bounded disturbances are handled via a one-step robustness buffer in the reachability test, not as hard constraints in \eqref{eq:cost}.

Our one–step question is simple: \emph{can we land on the look-ahead point in the next sample without breaking the limits?}
Because the limits in \eqref{eq:constraints} depend only on magnitudes, the component of the one–step position error along the straight line to the look–ahead point is the critical one: any acceleration orthogonal to that line does not reduce the axial gap in one step.

The closest point on the path is $\mathbf{p}_{\mathrm{ref}}(s_c)$.

\begin{lemma}%[Axial direction minimizes one–step position error]
\label{lem:axial}
Let 
\begin{equation}
    \mathbf r \coloneqq \mathbf p_{\mathrm{LA}} - \mathbf p_k - t_s\,\mathbf v_k, \qquad
    \hat{\mathbf e} = \frac{\mathbf r}{\|\mathbf r\|}, \quad \|\mathbf r\|>0.
    \label{eq:unit-direction}
\end{equation}
and the scalar projections
\begin{equation}
    p_{\mathrm{rel}}=\hat{\mathbf e}^{\!\top}\mathbf p_k,\quad
    p_{\mathrm{LA}} \,=\, \hat{\mathbf e}^{\!\top}\mathbf p_{\mathrm{LA}},\qquad
    v_{\mathrm{rel}}=\hat{\mathbf e}^{\!\top}\mathbf v_k .
    \label{eq:projection}
\end{equation}
%
%so that the one-step position error is \(\mathbf e_p(\mathbf u_k)=\mathbf r - \tfrac{1}{2}t_s^2\,\mathbf u_k\). Among all controls with fixed magnitude \(\|\mathbf u_k\|\), the quantity \(\|\mathbf e_p(\mathbf u_k)\|\) is minimized when \(\mathbf u_k\) is collinear with \(\mathbf r\).
The unconstrained minimizer of $\mathbf{e}_p$ is
\begin{equation}
    \mathbf u_k^* \,=\, \frac{2}{t_s^2}\,\mathbf r,
\end{equation}
%Furthermore, the axial component $u_{\mathrm{req}}=\hat{\mathbf e}^{\!\top}\mathbf u_k^*$ and considering \eqref{eq:projection} equals:
with a magnitude of,
\begin{equation}
    \hat{\mathbf{e}}^\top \mathbf{u}_k^* = u_{\mathrm{req}} \,=\, \frac{2\big( p_{\mathrm{LA}} - p_{\mathrm{rel}} - v_{\mathrm{rel}} t_s \big)}{t_s^2}.
    \label{eq:ureq}
\end{equation}
\end{lemma}

\begin{proof}
Expand
$\|\mathbf e_p(\mathbf u_k)\|^2=\|\mathbf r\|^2 - t_s^2\,\mathbf r^\top\mathbf u_k + \tfrac{1}{4}t_s^4\|\mathbf u_k\|^2$. 
%For fixed 
%$\|\mathbf u_k\|$, 
%this is minimized by maximizing 
%$\mathbf r^\top\mathbf u_k$, 
%which occurs when 
%$\mathbf u_k\parallel\mathbf r$. 
Setting 
$\nabla_{\mathbf u_k}\|\mathbf e_p\|^2=0$ 
yields 
$\mathbf u_k^{\*}=\tfrac{2}{t_s^2}\mathbf r$.
Projecting onto 
$\hat{\mathbf e}$ 
gives \eqref{eq:ureq}.
\end{proof}

\begin{remark}
If $\mathbf p_{\mathrm{LA}}=\mathbf p_k$ then $\hat{\mathbf e}$ can be any unit vector and
$u_{\mathrm{req}}=0$ for the axial check; the QP \eqref{eq:cost} handles the step.
\end{remark}

\begin{lemma}
With $\hat{\mathbf e}$ from Lemma~\ref{lem:axial}, any control of the form
$\mathbf u_k = u_{\mathrm{req}}\,\hat{\mathbf e}$
with $u_{\mathrm{req}}$ from \eqref{eq:ureq} makes the \emph{axial} one-step position error zero, i.e., 
$\hat{\mathbf e}^{\!\top}\mathbf e_p=0$. %,
%which was expected since 
%$u_{\mathrm{req}}\,\hat{\mathbf e}$ 
%is the global minimizer, and the global minimum of the error norm is zero. 
%That said, the quadratic objective $\|\mathbf e_p\|^2$ is convex in $\mathbf u_k$, hence the unconstrained minimizer $\mathbf u_k^*$ is unique.
\end{lemma}

\begin{proof}
Project the discrete update onto 
$\hat{\mathbf e}$: $p_{\mathrm{rel}}+v_{\mathrm{rel}}t_s+\tfrac{1}{2}u t_s^2=p_{\mathrm{LA}}$ 
with $u=\hat{\mathbf e}^{\!\top}\mathbf u_k$. 
Solving for $u$ gives \eqref{eq:ureq}, implying $\hat{\mathbf e}^{\!\top}\mathbf e_p=0$. 
Uniqueness comes from convexity of $\hat{e}^\top e_p$.% in $u_k$.
%Strict convexity follows from a positive-definite Hessian (scalar $u$ case) and from the quadratic form in $\mathbf u_k$ more generally.
\end{proof}

\begin{remark}
With acceleration/speed bounds, the QP generally cannot make the full vector one-step position error zero; that happens only if the unconstrained minimizer stands inside the feasible set.
Our one-step statement is therefore axial: $\hat{\mathbf e}^\top \mathbf e_p=0$;
velocity matching is handled by the $C$ term over subsequent steps.
\end{remark}

Using the discrete update with disturbances $\mathbf n_{p,k}$ and $\mathbf n_{v,k}$,
\begin{align}
\mathbf e_p
&= \mathbf p_{\mathrm{LA}} -
   \Big(\mathbf p_k + \mathbf v_k t_s + \mathbf n_{p,k} t_s
        + \tfrac{1}{2}\mathbf u_k t_s^2 + \tfrac{1}{2}\mathbf n_{v,k} t_s^2\Big)
     \label{eq:ep-noisy-raw}\\
&= \mathbf r - \tfrac{1}{2} t_s^2 \mathbf u_k
   \;-\; \mathbf n_{p,k} t_s \;-\; \tfrac{1}{2} \mathbf n_{v,k} t_s^2,
     \label{eq:ep-noisy-r}
\end{align}
and, by the triangle inequality,
\begin{align}
\|\mathbf e_p\|
&\le \big\|\mathbf r - \tfrac{1}{2} t_s^2 \mathbf u_k \big\|
    \;+\; \|\mathbf n_{p,k}\|\, t_s
    \;+\; \tfrac{1}{2}\|\mathbf n_{v,k}\|\, t_s^2      \\
&\le \big\|\mathbf r - \tfrac{1}{2} t_s^2 \mathbf u_k \big\|
    \;+\; \varepsilon_p t_s \;+\; \tfrac{1}{2}\varepsilon_v t_s^2.
\label{eq:ep-noisy-bound}
\end{align}

\paragraph*{One-step disturbance buffer.}
With noise bounded as $\|\mathbf n_p\|\le\varepsilon_p$ and $\|\mathbf n_v\|\le\varepsilon_v$,
the worst-case one-step change in \emph{position} is
\begin{equation}
    \Delta p_{\max} = \varepsilon_p\,t_s + \tfrac{1}{2}\,\varepsilon_v\,t_s^2.
    \label{eq:pos-buffer}
\end{equation}
Intuitively, $\Delta p_{\max}$ is how far noise might push us in one sample.
To guard against this within a single step, we convert the position buffer into an
\emph{acceleration} buffer:
\begin{equation}
    \sigma = \frac{2\,\Delta p_{\max}}{t_s^2} \;=\; \frac{2\varepsilon_p}{t_s} + \varepsilon_v.
    \label{eq:sigma}
\end{equation}

\begin{theorem}\label{thm:reach}%[Sufficient condition for robust one-step reachability]
If the required axial acceleration from \eqref{eq:ureq} satisfies
\begin{equation}
    \delta_k \;\coloneqq\; |u_{\mathrm{req}}| - \big(a_{\max} - \sigma\big) \;\le\; 0,
    \label{eq:reachability}
\end{equation}
then the look-ahead point is reachable in one step under the bounded disturbances.
That is, there exists a command $\mathbf u_k$ with $\|\mathbf u_k\|\le a_{\max}$ such that
$\hat{\mathbf e}^{\!\top}\mathbf e_p=0$.
\end{theorem}

\begin{proof}
By the \eqref{eq:ep-noisy-bound}, the noise can shift the one-step position by at most
$\Delta p_{\max}$. An extra axial term $\eta$ with $|\eta|\le\sigma$ (from \eqref{eq:sigma})
is enough to cancel that shift in one step. If $|u_{\mathrm{req}}|\le a_{\max}-\sigma$, choose
$\mathbf u_k=(u_{\mathrm{req}}+\eta)\,\hat{\mathbf e}$; this respects $\|\mathbf u_k\|\le a_{\max}$
and achieves axial landing. This proves sufficiency. The condition is not necessary:
if the actual noise is smaller than the worst case, landing may still occur even when $\delta_k>0$.
\end{proof}

\vspace{0em}
\subsubsection*{Reachable case ($\delta_k\le 0$)}
When the look-ahead target is reachable in one step, there is no need to chase velocity. The one-step prediction is
\begin{equation}
  \mathbf{p}_{k+1}=\mathbf{p}_k + t_s\,\mathbf{v}_k + \tfrac{1}{2}t_s^2\,\mathbf{u}_k.
  \label{eq:one-step-pos}
\end{equation}
To land exactly on the look-ahead point in the nominal (noise-free, unconstrained) model, set
$\mathbf{p}_{k+1}=\mathbf{p}_{\mathrm{LA}}$ and solve for $\mathbf{u}_k$:
\begin{equation}
    \mathbf{u}_k \;=\; \frac{2}{t_s^2}\Big(\mathbf{p}_{\mathrm{LA}} - \mathbf{p}_k - t_s\,\mathbf{v}_k\Big),
  \label{eq:one-step-acc}
\end{equation}
which yields $\mathbf e_p(\mathbf u_k)=\mathbf 0$.

If the limits in \eqref{eq:cost} are active, the QP returns the projection of \eqref{eq:one-step-acc} onto the feasible set, so a small residual in $\mathbf e_p$ may remain; any velocity mismatch is handled by the $C_k$ term over subsequent steps. Under $\delta_k\le 0$, there also exists an axial adjustment $\eta$ with $|\eta|\le\sigma$ (Theorem~\ref{thm:reach}) so that $\mathbf u_k=(u_{\mathrm{req}}+\eta)\,\hat{\mathbf e}$ keeps $\|\mathbf u_k\|\le a_{\max}$ and cancels the axial error despite bounded noise.

\begin{remark}
    The optimal solution,
    %\begin{equation}
     $   u_k^* = \frac{2}{t_s^2}\mathbf{r},$
    %\end{equation}
    also minimizes velocity error $||\mathbf{e}_v||^2$ if the current state matches the reference trajectory and $t_s$ is small.
\end{remark}

\begin{proof}
    The velocity error is,
    \begin{equation}
        e_v= v_{LA} - v_k - \frac{2}{t_s}(p_{LA}-p_k-t_sv_k).
    \end{equation}
    Let $s_0$ be the projected distance in \eqref{eq:closest-point}.
    By our premise $p_k = p_{ref}(s_0)$ and $v_k = v_{ref}(s_0)$, and $v_{ref}(s_0) = v_{LH}$ by definition.
    This implies
    \begin{equation}
        e_v = -2 \Big(\frac{p_{LA} - p_k}{t_s} - v_k\Big).
    \end{equation}
    In the limit $t_s\to0$, this becomes
    %\begin{equation}
        $e_v = -2(v_k - v_k) = 0.$
    %\end{equation}
\end{proof}

Remark 3 is significant, because it implies that the system is capable of perfectly tracking the trajectory in the deterministic case.

\vspace{0em}
\subsubsection*{Unreachable case ($\delta_k>0$)}
When the target (look-ahead point) is not reachable in one step, matching some of the \emph{phase} (velocity) helps us re-enter the tube smoothly. 

When the target is not one-step reachable, aligning only the axial component can leave
tangential mismatch. It is cleaner to work directly with the full vectors
\begin{equation}
    \mathbf e_p(\mathbf u)=\mathbf r-\tfrac{1}{2}t_s^2\mathbf u,
    \qquad
    \mathbf e_v(\mathbf u)=\mathbf d_v - t_s\mathbf u,\quad
    \mathbf d_v\coloneqq \mathbf v_{\mathrm{LA}}-\mathbf v_k.
\end{equation}
We minimize 
\begin{equation}
    J(\mathbf u;C)=\|\mathbf e_p(\mathbf u)\|^2+C\|\mathbf e_v(\mathbf u)\|^2
\end{equation}
subject to the same magnitude limits as the QP.

\begin{lemma}%[Unconstrained minimizer as a function of $C$]\label{lem:u-star-2d}
The unique unconstrained minimizer of $J(\mathbf u;C)$ is
\begin{equation}
    \mathbf u^{*}(C)
    = \frac{\mathbf r + \tfrac{2C}{t_s}\,\mathbf d_v}{\tfrac{1}{2}t_s^2 + 2C}.
\label{eq:u-star-C}
\end{equation}
If $\mathbf u^{*}(C)$ satisfies $\|\mathbf u^{*}(C)\|\le a_{\max}$ and 
$\|\mathbf v_k+t_s\mathbf u^{*}(C)\|\le v_{\max}$, it is also the QP solution.
\end{lemma}

\begin{proof}
Set $\nabla_{\mathbf u}J(\mathbf u;C)=\mathbf 0$:
$-t_s^2(\mathbf r-\tfrac{1}{2}t_s^2\mathbf u)-2Ct_s(\mathbf d_v-t_s\mathbf u)=\mathbf 0$,
then solve for $\mathbf u$.
\end{proof}

If either limit is active, the QP solution lies on the boundary. The stationarity
condition becomes
\begin{equation}
    \nabla_{\mathbf u}J(\mathbf u;C) + 2\lambda_1\,\mathbf u + 2\lambda_2\,t_s(\mathbf v_k+t_s\mathbf u)=\mathbf 0,
\label{eq:kkt-stationarity}
\end{equation}
with Lagrange multipliers $\lambda_1,\lambda_2\ge 0$ that enforce the constraints
$\lambda_1(\|\mathbf u\|^2-a_{\max}^2)=0$, 
$\lambda_2(\|\mathbf v_k+t_s\mathbf u\|^2-v_{\max}^2)=0$,
and the two norm constraints. 
Intuitively: start from $\mathbf u^{*}(C)$ and ``push'' it to the nearest point on the
active boundary that satisfies \eqref{eq:kkt-stationarity}.

We choose $C$ to balance how much position we fix versus how much velocity we line up in vector form. Evaluate the residuals at the current QP solution $\mathbf u_k$:
\begin{equation}
    \mathbf e_p=\mathbf r-\tfrac{1}{2}t_s^2\mathbf u_k,\qquad
\mathbf e_v=\mathbf d_v - t_s\mathbf u_k.
\end{equation}
so; 
\begin{equation}
    C_{\mathrm{KKT}} = \frac{t_s\,\|\mathbf e_p\|}{\,2\,\|\mathbf e_v\|}
    \qquad , \|\mathbf e_v\| \ne0.
\label{eq:C-KKT}
\end{equation}

To avoid large swings (e.g., when $\|\mathbf e_v\|$ is small), we bias and smooth:
\begin{equation}
    C_k=\min\Big\{(1-\beta)\,C_{k-1}+\beta\,\rho\,C_{\mathrm{KKT}},\;C_{\max}\Big\}.
    \label{eq:c-next}
\end{equation}

Here $\rho\in(0,1)$ when $\delta_k>0$ (favor matching some velocity), and $\rho=1$ when
$\delta_k\le 0$ (we can afford to emphasize position). This keeps one small QP, and
naturally reduces velocity mismatch once we are near the reachable tube.

\vspace{0em}
\begin{remark}
    If a safety or interaction trigger occurs (e.g., force threshold or E-stop), the controller enters a safe mode and pauses motion until the condition clears. After the clear signal, we resume the same online QP from the current $(\mathbf p_k,\mathbf v_k)$ \emph{without replanning the path}.
    The reachability check and the adaptive weight $C_k$ handle the transient: if the look-ahead target is reachable in one step, the QP cancels the position error; if not, it blends a velocity term to re-enter the tube smoothly under the bounds.
\end{remark}

\vspace{0em}
\subsection*{Offline reachability} \label{Subsec:offline-screen}
Before we run, we scan the fixed path once to flag geometry that would demand
too much axial acceleration in a single step.
We use the same building blocks as online, so the reasoning is consistent.

Choose the sample time $t_s$, limits $(a_{\max},v_{\max})$, and disturbance bounds
$(\varepsilon_p,\varepsilon_v)$. Compute the one-step acceleration buffer

\begin{equation}
    \sigma=\frac{2\varepsilon_p}{t_s}+\varepsilon_v
\quad\text{(from \eqref{eq:sigma})}.
\end{equation}
For each path sample $s$:
\begin{enumerate}[leftmargin=*]
\item Form the look-ahead pair from \eqref{eq:LA-p&v} with
$(\mathbf p_k,\mathbf v_k)\!\leftarrow\!(\mathbf p_{\mathrm{ref}}(s),\mathbf v_{\mathrm{ref}}(s))$.
\item Build $\mathbf r=\mathbf p_{\mathrm{LA}}-\mathbf p_k-t_s\mathbf v_k$, set
$\hat{\mathbf e}=\mathbf r/\|\mathbf r\|$ (if $\|\mathbf r\|=0$, skip this sample), and compute
\[
u_{\mathrm{req}}(s)=\frac{2}{t_s^2}\,\|\mathbf r\|.
\]
\item Compute the robust margin
\[
\delta(s)=|u_{\mathrm{req}}(s)|-\big(a_{\max}-\sigma\big).
\]
\end{enumerate} 

If $\delta(s)\le 0$, the look-ahead landing is feasible in one step even under worst-case noise (sufficient condition).
If $\delta(s)>0$, we mark that interval as ``unsafe''. 
Online, this is a cue to shrink look-ahead or apply a local speed cap; we do not replan the path.

This quick screen catches sharp curvatures or tight timing segments that would otherwise trigger frequent infeasible axial demands. Because it shares the same $\sigma$ and $u_{\mathrm{req}}$ logic as the online controller, what we see offline matches what we
enforce at runtime.

\section{Simulation Setup} \label{sec:setup}
\subsubsection*{Reference trajectory generation}
We build a 2D reference path from waypoints and fit a smooth cubic ($C^2$) spline through them. To simulate realistic task–space motions, we place the waypoints in random directions inside a fixed rectangular workspace of size $L_x \times L_y$. Let $n_{\mathrm{wp}}$ be the number of random waypoints. Set the first waypoint at the origin, $w_0=(0,0)^\top$. For $i=1,\dots,n_{\mathrm{wp}}$, draw a direction angle $\theta_i$ uniformly in $[0,2\pi)$ and place
\[
w_i \;=\; \begin{bmatrix} L_x\cos\theta_i \\ L_y\sin\theta_i \end{bmatrix}.
\]
This gives $n_{\mathrm{wp}}$ randomly oriented waypoints in the workspace $L_x \times L_y$. 
To avoid overly sharp turns or dense clusters, we rely on the cubic spline’s smoothness and chord-length timing to distribute curvature and pace along the path. 

For $i=1,\dots,n_{\mathrm{wp}}$, compute segment lengths
\begin{equation}
    d_i \;=\; \|p_i - p_{i-1}\|, \qquad s_i=\sum_{j=1}^{i} d_j,\quad s_0=0,
\end{equation}
Then assign time to each segment in proportion to its length (“chord–length timing”). Let $p_i$ be the $i$th waypoint and be the segment length. With total duration $T$, the time assigned to segment $i$ is
\begin{equation}
    t_i \;=\; T \cdot \frac{d_i}{\sum_{j=1}^N d_j}.
\end{equation}

The cumulative times $\{0,\; t_1,\; t_1{+}t_2,\dots\}$ serve as spline breakpoints. This spreads curvature over time and reduces acceleration/jerk spikes.

From the spline we evaluate continuous position, velocity, and acceleration, denoted $\mathbf p_{\mathrm{ref}}(\cdot)$, $\mathbf v_{\mathrm{ref}}(\cdot)$, and $\mathbf a_{\mathrm{ref}}(\cdot)$. For fast lookup we resample them on a dense grid to form arrays used by the controller. Online, the tracker finds the closest point on the path and the look-ahead target using \eqref{eq:closest-point}–\eqref{eq:LA-p&v}; it does not follow a fixed time law.

\begin{remark}%[Smoothness and jerk] \label{}
Cubic fitting plus chord-length timing yields smooth position and velocity with moderated acceleration. If even smoother motion is needed, $T$ can be increased slightly or the waypoint set can be lightly smoothed before spline fitting.
\end{remark}

%\subsubsection*{Performance Metrics}
We compare the controllers' performance using %The controllers are compared using:
%\begin{itemize}
%    \textbf{RMSE:}
\begin{equation}
    \mathrm{RMSE} = \sqrt{\frac{1}{N}\sum_{k=1}^{N} \big\|\mathbf{x}(k)-\mathbf{x}_{\mathrm{LA}}(k)\big\|^{2}},
    \label{eq:RMSE}
\end{equation}
where $\mathbf{x}$ corresponds to either the position or velocity, and $N$ is the number of runs.
We also consider the mean rachability margin,
%Here $\mathbf{x}\in\{\mathbf{p},\mathbf{v}\}$ and $\mathbf{x}_{\mathrm{LA}}\in\{\mathbf{p}_{\mathrm{LA}},\mathbf{v}_{\mathrm{LA}}\}$ are the measured position/velocity and $N$ counts only \emph{moving} samples (the freeze interval is excluded). We compute this RMSE twice: once with $(\mathbf{x},\mathbf{x}_{\mathrm{LA}})=(\mathbf{p},\mathbf{p}_{\mathrm{LA}})$ and once with $(\mathbf{x},\mathbf{x}_{\mathrm{LA}})=(\mathbf{v},\mathbf{v}_{\mathrm{LA}})$.
%
%Mean reachability margin:
%\begin{equation}
    $\bar{\delta} = \frac{1}{N} \sum_{k=1}^N \delta_k%\left[ |u_{\mathrm{req}}(k)| - a_{\max} \right],
    \label{eq:delta-mean}$
%\end{equation}
using \eqref{eq:reachability}.
%with \(u_{\mathrm{req}}\) from~\eqref{eq:ureq}.
%\end{itemize}
Metrics are computed over the motion period only, excluding the freeze interval.

%\subsubsection*{Controller parameters (values used in simulation)}
The simulation parameters are available online\footnote{\url{https://github.com/Hossein-ghd/Reachability-Check-QP}}.
%Table~\ref{tab:params} summarizes the parameters we used in the simulations.
The values of $t_s$, $v_{\max}$, and $a_{\max}$ come directly from the UR5e robot limits, since our controller runs in Cartesian space and we wanted the setup to be realistic. The other parameters such as $\rho$, $\beta$, and the bounds on $C$ were chosen from experiments and tuning. They control how fast the adaptive weight reacts and how conservative the reachability margin is. In this way, the table gives both the physical constraints of the system and the control logic settings we used.

%\begin{table}[t]
%\centering
%\caption{Simulation values used by the controller}
%\label{tab:params}
%\begin{tabular}{lll}
%\hline
%Symbol & Value & Ref. \\
%\hline
%$t_s$              & $0.008$\,s (125\,Hz)              & — \\
%$v_{\max}$         & $1.0$\,m/s                       & \eqref{eq:constraints} \\
%$a_{\max}$         & $2.5$\,m/s$^2$                   & \eqref{eq:constraints} \\
%$s_{\mathrm{LA}}$  & $\ge 1$\,mm                       & \eqref{eq:lookahead} \\
%$\rho_{\mathrm{in}}$  & $1.0$                         & \eqref{eq:rho} \\
%$\rho$ ($\delta_k> 0$) & $0.4$                        & \eqref{eq:rho} \\
%$\beta$            & $0.65$                          & \eqref{eq:c-next} %\\
%$C_{\min}$         & $10^{-6}$                           & \eqref{eq:c-next} \\
%$C_{\max}$         & $10^{3}$                         & \eqref{eq:c-next} \\
%$C_0$               & $10^{-3}$                     & \eqref{eq:c-next}\\

%\hline
%\end{tabular}
%\end{table}

\section{Results} \label{sec:results}
We evaluate the proposed QP+reachability tracker against a geometric pure–pursuit (PP) baseline under the same limits and sampling period $t_s$. Both methods use the same closest–point and look–ahead rules \eqref{eq:closest-point}–\eqref{eq:LA-p&v}; PP steers toward the look–ahead with a PD-like term and enforces bounds only by post-hoc clipping. Each trial contains one $1.0$\,s freeze inserted at a uniformly random time, as described in Sec.~\ref{sec:setup}. For aggregate plots we remove the frozen portion and normalize the remaining “moving” time to $T\in[0,1]$.

Fig.~\ref{fig:qp_pp_behaviour} shows one representative run on a random spline; colors are fixed across all comparison plots: blue = QP+reachability, orange = pure pursuit. In this plot, the reference (black dashed) is tracked by the proposed method and PP. The QP+reachability trajectory stays close to the curve and rejoins smoothly after the pause, whereas PP overshoots and accumulates timing mismatch around high–curvature segments. The look-ahead targets and current poses (markers inherit the line color) illustrate that the proposed controller lands near the look-ahead in one step when reachable and otherwise aligns velocity to re-enter the tube. The proposed method stays close to the reference and rejoins the path smoothly after a pause.

\begin{figure}[ht]
    \centering
    \includegraphics[width=0.95\columnwidth]{{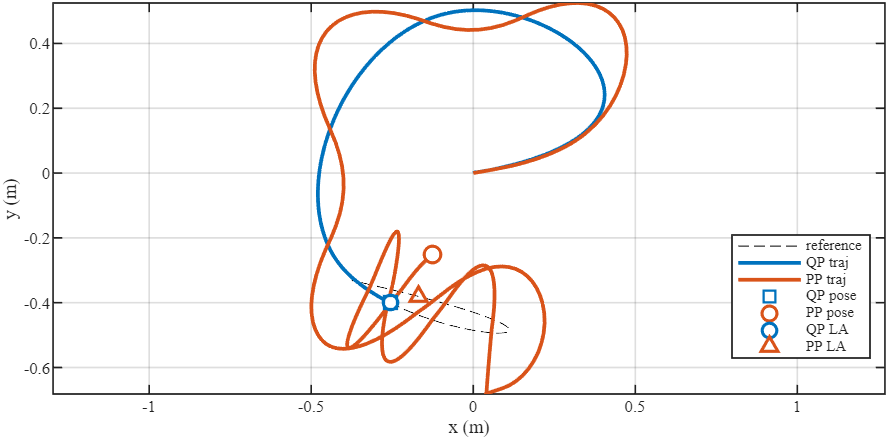}}
    \caption{Single-run comparison on a random spline for pure pursuit and our proposed controller.}
    \label{fig:qp_pp_behaviour}
\end{figure}

Over $50$ random paths (with one freeze each), the proposed method achieves uniformly lower tracking error than PP. %Fig.~\ref{fig:rmse} reports RMS position (top) and velocity (bottom) versus normalized moving time.
The separation is most visible around mid-run, where many trials resume from a pause or encounter sharp turns. This supports the role of the one-step reachability test (to avoid infeasible axial demands) and the adaptive velocity term (to reduce timing mismatch when the target is not reachable in one step).

%\begin{figure}[ht]
%    \centering
%\includegraphics[width=0.9\columnwidth]{{Figs/RMS Position and Velocity Errors (QP vs PP).png}}
%    \caption{Tracking RMSEs (median across runs) vs.\ normalized moving time.}
%    \label{fig:rmse}
%\end{figure}

We quantify how often the commanded motion would exceed the acceleration bound using the one-step reachability margin $\delta=|u_{\mathrm{req}}|-a_{\max}$ from~\eqref{eq:ureq}. Fig.~\ref{fig:mean_delta_hist} shows the distribution of the per-run mean margin (moving samples only). QP+reachability is concentrated well below $0$ (inside the limit), while PP exhibits high positive values, consistent with enforcing bounds only by a posteriori saturation. So, the proposed method stays inside the limit with margin; PP often demands super-limit acceleration.

\begin{figure}[ht]
    \centering
    \includegraphics[width=0.85\columnwidth]{{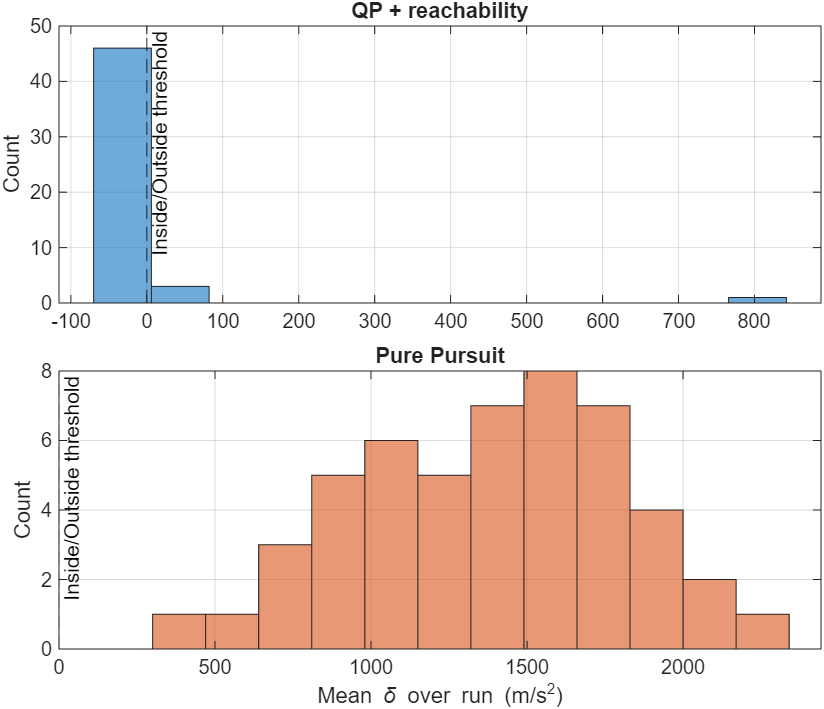}}
    \caption{Distribution of per-run mean reachability margin $\delta$ (moving samples). Dashed line: $\delta=0$. Top: QP+reachability; bottom: pure pursuit. Note the different axis scales.}
    \label{fig:mean_delta_hist}
\end{figure}

The temporal evolution in Fig.~\ref{fig:delta_time} further highlights the difference. For QP+reachability (top), $\delta(T)$ remains negative with small magnitude over the entire motion and only brief near-zero crossings around sharp geometric transitions; for PP (bottom), $\delta(T)$ is frequently and substantially positive, indicating repeated infeasible axial demands ($a_{\max}$) that must be clipped by saturation. A short initial transient appears because runs start from rest.

\begin{figure}[ht]
    \centering
    \includegraphics[width=0.9\columnwidth]{{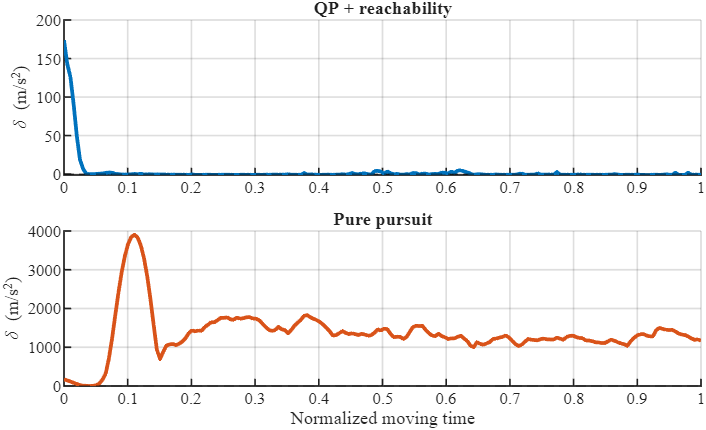}}
    \caption{Mean reachability margin $\delta(T)$ (average across runs) vs.\ normalized moving time. Top: QP+reachability; bottom: pure pursuit.}
    \label{fig:delta_time}
\end{figure}

Finally, Fig.~\ref{fig:pv_time} illustrates pause–and–resume with the proposed controller alone. It plots $x$/$y$ position (top) and velocity (bottom) for a $10$\,s run where the $1$\,s freeze interval is \emph{removed} (yielding $9$\,s on the time axis). The vertical dashed line marks the instant at which the freeze occurred in real time; the short gap is removed so the curves show only ``moving'' samples.
The short disturbance at the dashed line (when the freeze occurred in real time) comes from stitching out the one-second stop; on either side, the time histories are smooth. This plot demonstrates the intended logic: (i) brake-to-rest during a stop; (ii) resume with the same path and per-sample feasibility; and (iii) quickly restore timing mismatch via the adaptive position/velocity trade–off (inside the tube, the update drives $C$ toward its minimum).

\begin{figure}[ht]
    \centering
    \includegraphics[width=0.9\columnwidth]{{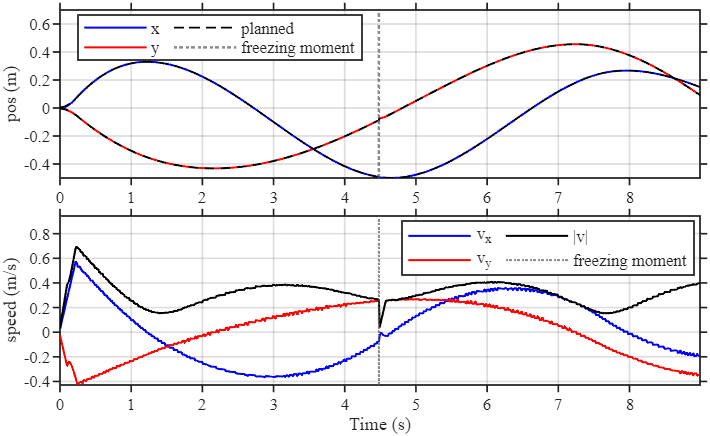}}
    \caption{Per-axis position and velocity with the $1$\,s freeze removed (time axis shows $9$\,s of motion).}
    \label{fig:pv_time}
\end{figure}

Table~\ref{tab:rmse} reports aggregate RMSE over the moving window (mean\,$\pm$\,std across 50 runs). Relative to PP, QP+reachability reduces position RMSE by $\approx\!95.5\%$ and velocity RMSE by $\approx\!81.7\%$.%, consistent with Fig.~\ref{fig:rmse}.

\begin{table}[ht]
    \centering
    \caption{Aggregate RMSE over moving time (mean\,$\pm$\,std across 50 runs).}
    \label{tab:rmse}
    \begin{tabular}{lcc}
    \hline
    Metric & QP+reachability & Pure pursuit \\
    \hline
    Position RMSE [m]   & 0.0040 $\,\pm\,$ 0.0145 & 0.0896 $\,\pm\,$ 0.0275 \\
    Velocity RMSE [m/s] & 0.13   $\,\pm\,$ 0.25   & 0.71   $\,\pm\,$ 0.20   \\
    \hline
    \end{tabular}
\end{table}

Overall, when the look–ahead point is one-step reachable, the position-dominant behavior (small $C$) keeps accelerations within bounds while snapping toward the target; when it is not, the adaptive weight biases the QP to match velocity along the axial direction, reducing timing mismatch and shortening time spent outside the tube. Because all decisions are in output space with explicit speed/acceleration constraints, performance is not sensitive to path timing and remains feasible by construction, explaining the consistently lower errors and strongly negative margins observed across figures.

\section{Conclusion and Outlook} \label{sec:conclusion}
We proposed an output–space tracker that couples an offline feasibility screen with a per-sample QP enforcing speed/acceleration bounds. A one-step reachability margin with an adaptive position/velocity weight selects snap-to-position versus timing match and handles freeze–resume without replanning. Across random paths with a single freeze, the method reduced RMS position/velocity errors and acceleration-limit violations relative to pure pursuit at real-time rates. Looking ahead, we will (i) handle state- and path-dependent limits, (ii) extend to higher-order (jerk-limited) dynamics, and (iii) integrate a safety envelope (e.g., barrier constraints) within the same QP.

\bibliographystyle{IEEEtran}
\bibliography{references}

\begin{thebibliography}{99}

\bibitem{hogan1985impedance}
N.~Hogan, ``Impedance control: An approach to manipulation: {P}art {II}---Implementation,''
\emph{Journal of Dynamic Systems, Measurement, and Control}, vol.~107, no.~1, pp.~8--16, 1985.

\bibitem{mohammadi2024development}
V.~Mohammadi, R.~Shahbad, M.~Hosseini, M.~H.~Gholampour, S.~Shiry~Ghidary, F.~Najafi, and A.~Behboodi,
``Development of a two-finger haptic robotic hand with novel stiffness detection and impedance control,''
\emph{Sensors}, vol.~24, no.~8, p.~2585, 2024.

\bibitem{chemweno2020orienting}
P.~Chemweno, L.~Pintelon, and W.~Decre,
``Orienting safety assurance with outcomes of hazard analysis and risk assessment: A review of the {ISO} 15066 standard for collaborative robot systems,''
\emph{Safety Science}, vol.~129, p.~104832, 2020.

\bibitem{pham2018new}
H.~Pham and Q.-C.~Pham,
``A new approach to time-optimal path parameterization based on reachability analysis,''
\emph{IEEE Transactions on Robotics}, vol.~34, no.~3, pp.~645--659, 2018.

\bibitem{coulter1992}
R.~C.~Coulter, ``Implementation of the pure pursuit path tracking algorithm,'' Tech. Rep., 1992.

\bibitem{paden2016survey}
B.~Paden, M.~{\v{C}}{\'a}p, S.~Z.~Yong, D.~Yershov, and E.~Frazzoli,
``A survey of motion planning and control techniques for self-driving urban vehicles,''
\emph{IEEE Transactions on Intelligent Vehicles}, vol.~1, no.~1, pp.~33--55, 2016.

\bibitem{Berscheid2021JerklimitedRT}
L.~Berscheid and T.~Kr{\"o}ger,
``Jerk-limited real-time trajectory generation with arbitrary target states,''
\emph{arXiv preprint} arXiv:2105.04830, 2021.

\bibitem{kermanshah2024control}
M.~Kermanshah, L.~E.~Beaver, M.~Sokolich, S.~Das, R.~Weiss, R.~Tron, and C.~Belta,
``Control of microrobots using model predictive control and {G}aussian processes for disturbance estimation,''
\emph{arXiv e-prints}, arXiv:2406.xxxx, 2024.

\bibitem{norouzi2025novel}
M.~Norouzi, M.~Zhou, and C.~Yuan,
``A novel adaptive formation control strategy for teams of unmanned vehicles under complete dynamic uncertainty,''
in \emph{AIAA Aviation Forum and ASCEND}, 2025, p.~3192.

\bibitem{ren2008distributed}
W.~Ren and R.~W.~Beard, \emph{Distributed Consensus in Multi-Vehicle Cooperative Control: Theory and Applications}.
Springer, 2008.

\bibitem{gholampour2025mass}
H.~Gholampour, J.~E.~Slightam, and L.~E.~Beaver,
``Mass-adaptive admittance control for robotic manipulators,''
in \emph{5th Modeling, Estimation and Control Conference (MECC)}, 2025 (to appear).

\end{thebibliography}

\end{document}